\documentclass[letterpaper, 10pt, conference]{ieeeconf}
\IEEEoverridecommandlockouts 
\usepackage{amsmath,amscd,amssymb,amsgen,amsfonts,amsbsy}
\usepackage{mathrsfs}
\usepackage[utf8x]{inputenc}
\usepackage{color}
\usepackage{textcomp}
\usepackage{float}
\usepackage{latexsym,graphicx}

\usepackage{tikz}
\usetikzlibrary{automata,arrows,positioning,calc}
\usepackage{url}
\usepackage{cite}
\usepackage{algorithmic}
\usepackage[ruled,lined]{algorithm2e}

\newcommand{\prob}[1]{\mathsf{Pr}\left( #1 \right)}

\newcommand{\remove}[1]{}

\newcommand{\comments}[1]{}

\newcommand{\qed}{\hfill $\square$}

\newtheorem{lemma}{Lemma}

\newtheorem{theorem}{Theorem}

\newtheorem{definition}{Definition}

\usetikzlibrary{shapes.geometric}

\tikzset{
	buffer/.style={
		draw,
		shape border rotate=120,
		isosceles triangle,
		isosceles triangle apex angle=69,
		fill=white,
		node distance=10cm,
		minimum height=10em
	}
}

\makeatother

\title{Indexability and Rollout Policy for Multi-State Partially Observable Restless Bandits}
	
\author{
 \begin{tabular}{ccc}
   Rahul Meshram  & Kesav Kaza 
   \\
    Deptt. of Elect. Comm. Engg.    & Deptt. of Elecl.  Engg. 
    \\
    IIIT Allahabad       & Polytechnique Montreal 
    \\
    INDIA        &    CANADA  
  \end{tabular}
 }

\begin{document}

\maketitle

\begin{abstract}
  Restless multi-armed bandits with partially observable states has applications in communication systems, age of information and recommendation systems.   In this paper, we study multi-state partially observable restless bandit models. We consider three different models based on information observable to decision maker---1) no information is observable from actions of a bandit 2) perfect information from bandit is observable only for one action on bandit, there is a fixed restart state, i.e., transition occurs from all other states to that state 3) perfect state information is available to decision maker for both actions on a bandit and there are two  restart state for two actions.   We develop the structural properties. We also show a threshold type policy and indexability for model $2$ and $3.$ We present Monte Carlo (MC) rollout policy. We use it for whittle index computation in case of model $2.$ We obtain the concentration bound on value function in terms of horizon length and number of trajectories for MC rollout policy. We derive explicit index formula for model $3.$ We finally describe Monte Carlo rollout policy for model $1$ when it is difficult to show indexability.  We demonstrate the numerical examples using myopic policy, Monte Carlo rollout policy and Whittle index policy. We observe that Monte Carlo rollout policy is good competitive policy to myopic.  
  
\end{abstract}

\section{Introduction}
 Restless multi-armed bandits with partially observable states have been recently found applications in online recommendation systems \cite{Meshram15}, opportunistic communication systems \cite{Zhao08,Zhao07,LiuZhao10}, machine maintenance \cite{Abbou19}, age of information, \cite{Shao20}. Restless multi-armed bandits (RMABs) are class of sequential decision problem with multiple independent Markov processes which are coupled via number of independent process that are activated simultaneously, \cite{Whittle88}. In a partially observable model, states of Markov chains are not observable at time of decision making but signals are observable. The solution of RMAB are computationally challenging and known to be PSPACE Hard problem, \cite{Papadimitriou99}. In fact a popular heuristic Whittle index based policy have been studied and it has shown to be asymptotically optimal, \cite{Ouyang12}. The essential idea of index policy is to decouple the independent Markov processes (arms) by solving relaxed constrained problem with Lagrangian  method. Later one need to show indexability for each processes and has to provide computational method for index which maps the state of each process to a real number. The process (arm) with the highest index is played at each time instant.   
 
Most of RMAB problems with  partially observable states are studied for two state model with various assumptions on transition probabilities, reward structure and observation probabilities, \cite{Meshram15,Ny08,Zhao08,Zhao07,LiuZhao10,Ahmad09,Meshram18}. Much less  attention is given to more than two state model. Multi-state partially observable RMAB has been studied in \cite{Ouyang14,Wang13,Larranga18,Shao20}. In \cite{Ouyang14,Wang13},  the optimality of myopic policy is shown under specific model assumption for identical communication channels. 
In \cite{Shao20}, authors have proposed and analyzed greedy policy for age of information problem.  In \cite{Larranga18}, authors have studied a pilot allocation problem in wireless networks over partially observable fading channel with approximation on multi state model. Further, they analyzed index policy and asymptotic optimality is proved.  To derive obtain indexability, one require to study a single armed bandit model and it is partially observable Markov decision process (POMDP). The properties of POMDP are derived in \cite{Lovejoy87}.

In this paper, we study partially observable RMAB with  more than two state model. We consider three different models based on information observable to decision maker. In first model we study with no state is observable for any actions. In second model the decision maker can observe the perfect state for one of the actions. In third model we assume that decision maker observes perfect state for both actions. We obtain structural properties and discuss about indexability for these models.  We discuss simulation based MC rollout policy.  In first model, indexability is very difficult to obtain and hence use rollout policy. In second model, we show indexability but difficult to derive index, this motivated rollout policy based index computation method. We obtain the concentration bound for rollout policy with threshold type structure. In third model, we show indexability and derive explicit index formula. Finally we illustrate performance of proposed policy using numerical examples.

The paper is organized as follows. We present model descriptions and preliminaries   in Section~\ref{sec:Model}. The structural properties and indexability are developed in Section~\ref{sec:Structural-Prop}. Monte Carlo rollout policy is discussed in Section~\ref{sec:Monte-Carlo-Rollout}. Numerical examples and discussion are presented in Section~\ref{sec:numerical-discussion}.



\section{Model Description}
\label{sec:Model}
Consider $N$ partially observable restless multi-armed bandits, where 
$\mathbf{M}_{i}= \{\mathcal{S}_i, \mathcal{A}_i, \mathcal{P}_i, \mathcal{R}_i, \mathcal{O}_i, \mathcal{Q}_i, \beta\},$ $i=1,2,\cdots,N.$ 
Let $\mathcal{S}_i$ be the state space,  $\mathcal{S}_i = \{1,2, \cdots, n\},$ $\mathcal{A}_i = \{0,1\}$ is action space, $\mathcal{P}_i = \{[[p_{jk}^a]]\}_{\{a \in \mathcal{A}\}}$ is the transition probability matrix and $p_{jk}^a$ is the transition probability from state $j$ to $k$ when action $a$ is applied. The decision maker (DM) does not observe the state of systems but makes his decisions based on the information obtained via evolution of states. Based on this observed information, the decision maker selects action $a_{t,i} \in \mathcal{A}_i$ at time $t=1,2, \cdots.$ The state of system $i$ at time $t$ is denoted by $s_{t,i} \in \mathcal{S}.$ A DM receives a real valued reward $r_i(j,a)$ if $a_{t,i} =a$ and $s_{t,i} = j.$ 
The system $i$ make transition to state $s_{t+1,i},$ and  $p_{jk}^a = \prob{s_{t+1,i} =j~|~s_{t,i} = i, a_{t,i} =a }.$
A DM perceives one of finite number of messages. Assume that   $\mathcal{O} = \{1,2,3, \cdots, K \}$ represents  the set of messages\footnote{Example is a google news recommendation system, where different messages correspond to actions of a user---like, dislike, watch later etc. The user takes different actions with some probability based on user interest state. This  generates reward to RS based on user behavior}.  If the message $k \in \mathcal{O}$ is observed with known probability from state $j$ under action $a$ for systen $i$ and this is denoted by $q_{i,jk}^a = \prob{k~|~s_{t,i}=j, a_{t,i} = a}.$ Thus $\mathcal{Q}_i = [[q_{i,jk}^{a}]]_{\{a \in \mathcal{A}\}}.$ 
The discount parameter is denoted by $\beta.$
Each bandit evolves in discrete time steps. 

An  infinite horizon discounted problem with  a policy  $\phi $ is given as follows.
\begin{eqnarray}
V_{\phi}(s) = 
\mathrm{E}_{\phi}\left(\sum_{t=0}^{\infty} \sum_{i=1}^{N} \beta^{t} r_i(s_{t,i},a_{t,i})  \right).
\label{eqn:T-horizon-valf-RMAB}
\end{eqnarray}
There is an activation constrained on bandits, i.e., $\sum_{i=1}^{N} a_{t,i} = 1.$ The policy $\phi: H_t \rightarrow \{1,2,,3, \cdots,N\},$ where $H_t$ denotes the history upto time $t,$ and $H_t:=\{a_1,o_1, \cdots, a_{t-1},o_{t-1}  \}.$
The Markov stationary deterministic policy is studied. 
$V_{\phi}(s)$ is the value function for given initial state $s.$ 
DM's goal is to choose the strategy $\phi$ to optimize $V_{\phi}(s),$ subject to constraint $\sum_{i=1}^{N} a_{t,i} = 1.$ Thus, the optimal value function is denoted by $V^*.$
 The discounted relaxed constrained problem using Lagrangian method is written as follows. 
\begin{eqnarray}
V_{\phi}(s) &=& 
\mathrm{E}_{\phi}\left(\sum_{i=1}^{N} \sum_{t=0}^{\infty}  \beta^{t} \left[  r_i(s_{t,i},a_{t,i}) + W (1 - a_{t,i}) \right] \right). \nonumber  \\
V^*(s) &=& \max_{\phi \in \Phi}  V_{\phi}(s).
\label{eqn:T-horizon-valf-RMAB-relaxed}
\end{eqnarray}
Here, $\Phi$ is the space of all Markov stationary deterministic policies.

\subsection{A single-armed restless bandit and preliminaries}
\label{sec:model-POMDP}
In this section, a single armed bandit with partially observable state is discussed and we remove dependence of arm on $i$ for notation simplicity. A single armed restless bandit  is a special case of partially observable Markov decision processes (POMDPs). 
We can rewrite problem in~\eqref{eqn:T-horizon-valf-RMAB-relaxed} for partially observable with belief $\pi.$ 
The DM only observes messages (signals) but no state information. The DM maintains initial belief as prior $\pi \in \Pi(\mathcal{S}),$ where  $\Pi(\mathcal{S}) = \{ \pi =(\pi(1), \pi(2), \cdots, \pi(n))~\vert~ \sum_{j=1}^{n} \pi(j) =1,  0 \leq \pi(j) \leq 1, \text{ for all } j \in \mathcal{S} \}$ is belief space and   $\pi(j)$ is probability of state being $j,$ i.e., $s = j.$ Based on initial belief $\pi,$ the value function under policy $\phi$ is 

{\small{
\begin{eqnarray*}
V_{\phi}(\pi) = 
\mathrm{E}_{\phi}\left( \sum_{t=0}^{\infty}  \beta^{t} \left[  \sum_{j=1}^{n} r(s_{t} = j,a_{t}) \pi(j) + W (1 - a_{t,i}) \right] \right).
\end{eqnarray*}
}}

The  DM optimizes the value function   and it is given by 
\begin{eqnarray}
V^{*}(\pi) = \max_{\phi \in \Phi} V_{\phi}(\pi).
\label{eqn:T-horizon-opt-valf-belief}
\end{eqnarray} 
From \cite{BertsekasV195,BertsekasV295}, we know that the information observed in the history $H_t$ is captured in form of belief $\pi_t \in \Pi(\mathcal{S}),$    $\pi_t$ is the Bayesian posterior over states given history 
\begin{eqnarray*}
\pi_{t}(j) &=& \prob{s_t = j~|~H_t} \\
\pi_{t}(j) &=& \prob{s_t = j~|~\pi_{t-1},o_t =k,a_t = a}\\
\pi_t &=& (\pi_{t}(1),\pi_{t}(2), \cdots,\pi_{t}(n) ).
\end{eqnarray*}
This is  shown to be sufficient information which captures all history upto $t.$  
Note that there are two actions are available to a single armed bandit---play or not play. Corresponding to this, there are actions dependent transition probabilities. We study the following models for a single armed bandit based on transition probabilities  and information observed from each action. 


\subsubsection{Model $1$}
In this model, a decision maker does not observe  state from both actions.  
This is an example of two action POMDP, where action $a=1$ provides a signals and other action $a=0$ provides no information to decision maker. 
For action $a = 1,$ DM observes a signal $k$ and the posterior belief is computed and the computations are as follows. 
Let $\xi(j,k~|~\pi, a)$ be the probability that the message $k$ is received from state $j$ given prior $\pi_t$ and action $a,$ and $\xi(j,k~|~\pi, a) = \sum_{i\in S} \pi_t(i)  p_{i,j}^a q_{i,k}^a.$ Define $\sigma(k~|~\pi,a)$ is the probability of observing message $k$ given prior $\pi_t$ and action $a.$ It is given by 
\begin{eqnarray*}
\sigma(k~|~\pi_t,a) &=& \sum_{j=1}^{n} \xi(j,k~|~\pi_t, a) \\ 
&=& \sum_{j \in S} \sum_{i\in S} \pi_t(i)  p_{i,j}^a q_{i,k}^a. 
\end{eqnarray*}
The Bayesian posterior given prior $\pi_t$ and action $a$ and signal $k$ is denoted by  $\Gamma(\pi_t,a,k),$ and $\Gamma_j(\pi,a,k) = \frac{\xi(j,k~|~\pi_t, a)}{\sigma(k~|~\pi_t,a)},$ $\Gamma(\pi_t,a,k) = (\Gamma_1(\pi_t,a,k), \cdots, \Gamma_n(\pi_t,a,k)) \in \Pi(\mathcal{S}).$  Then
\begin{eqnarray*}
	\pi_{t+1}(l) = \Gamma_l(\pi_t,a_t, o_t = k) 
	&=& \frac{\sum_{i\in S} \pi_t(i)  p_{i,l}^a q_{i,k}^a}{\sum_{j \in S} \sum_{i\in S} \pi_t(i)  p_{i,j}^a q_{i,k}^a}.
\end{eqnarray*}
When action $a =0,$ no signal is observed and hence 
the posterior belief $\pi_{t+1} = \pi_t P^0.$ 

Let $B(\mathcal{S})$ be the set of bounded real valued functions on $\Pi(\mathcal{S}).$ Define function $g:\Pi(\mathcal{S}) \times \mathcal{A} \times B(\mathcal{S}) \rightarrow \mathcal{R}$ and we can write $g(\pi,a,V)$ as function of immediate reward and future value function, thus 
\begin{eqnarray*}
g(\pi,a = 1,V) &=& \sum_{j=1}^{n} \pi(j) r(j,a=1) + \beta \sum_{k \in \mathcal{O}} \sigma(k~|~\pi,a) \\ & & \times  
V(\Gamma(\pi,a,k)) \\
g(\pi,a = 0,V) &=& \sum_{j=1}^{n} \pi(j) r(j,a=0) + W + \beta  V(\pi P^0) 
\end{eqnarray*} 
for $\pi \in \Pi(\mathcal{S}),$ $a \in \mathcal{A}$ and $V \in B(\mathcal{S}).$
$P^0$ is the transition probability for not playing arm.

An optimal dynamic programming algorithm is given as follows.
\begin{eqnarray}
V^*(\pi) = \max_{a \in \mathcal{A}} g(\pi,a,V^{*}).
\label{eqn:opt-dynamic-infinite}
\end{eqnarray}
It is difficult to claim indexability for this model and apply  index policy. Hence we study MC rollout policy in next section. 

\subsubsection{Model $2$}
In this model, a decision maker takes action $a=1,$ it just provide signals but does not provide any perfect information about state. The action $a=0$  gives perfect state information. Moreover, the transition occurs to  a fixed state $m$ which is restart state. Then the dynamic program is given as follows. 
\begin{eqnarray*}
	g(\pi,a = 1,V) &=& \sum_{j=1}^{n} \pi(j) r(j,a=1) + \beta \sum_{k \in \mathcal{O}} \sigma(k~|~\pi,a) \\ & & \times  
	V(\Gamma(\pi,a,k)) \\
	g(\pi,a = 0,V) &=& \sum_{j=1}^{n} \pi(j) r(j,a=0) + W + \beta  V(e_m) 
\end{eqnarray*} 
where  $e_m = [0,0, \cdots, 1,0, \cdots, 0]^T,$ $1$ is for state $m.$ The transition probability matrix of not playing action is $P^0,$ and $m$th column of it is a unit vector, i.e., all elements are $1$ and remaining columns are zero vectors. An optimal dynamic programming algorithm is given by  
\begin{eqnarray}
V^*(\pi) = \max_{a \in \mathcal{A}} g(\pi,a,V^{*}).
\label{eqn:opt-dynamic-infinite-model2}
\end{eqnarray}
In next section, we show that a bandit is indexable and but it is difficult to obtain  closed form expression of Whittle index. 

\subsubsection{Model $3$}
In this model, we  further relax   assumptions stated in previous models. We assume that state is perfectly observable for both actions. Moreover for action $a=1,$ transition from state $i$ to fixed state $m_1 \in \mathcal{S}$ occurs with probability $1,$ $i =1,2, \cdots, n.$ Similarly, for action $a=0$ a state transition from state $i$ to a fixed state $m_2 \in \mathcal{S}$ occurs with probability $1.$ 
The dynamic program is  
\begin{eqnarray*}
	g(\pi,a = 1,V) &=& \sum_{j=1}^{n} \pi(j) r(j,a=1) + \beta   
	V(e_{m_1}) \\
	g(\pi,a = 0,V) &=& \sum_{j=1}^{n} \pi(j) r(j,a=0) + W + \beta  V(e_{m_2}) 
\end{eqnarray*} 
where  $e_{m_1} = [0,0, \cdots, 1,0, \cdots, 0]^T,$ $1$ is at position $m_1.$
$e_{m_2} = [0,0, \cdots, 1,0, \cdots, 0]^T,$ $1$ is at position $m_2.$ 
An optimal dynamic programming algorithm is 
\begin{eqnarray}
V^*(\pi) = \max_{a \in \mathcal{A}} g(\pi,a,V^{*}).
\label{eqn:opt-dynamic-infinite-model3}
\end{eqnarray}

We will show  that a bandit is indexable and even obtain the closed form expression of Whittle index. 

\section{Structural results and Indexability}
\label{sec:Structural-Prop}
In this section we provide structural results, indexability of a restless bandits and derive index formula.
We derive two key results---monotonicity of optimal value functions and threshold type policy.

\subsection{Structural Properties}

\begin{lemma}[Convexity of value function]
	 For infinite horizon problem, the optimal value function $V^*(\pi)$ is convex in $\pi$ for $\pi \in \Pi(S).$   
	\label{lemma:convexity}
\end{lemma}
Proof of this result using induction method, and it uses \cite[Lemma $2$]{Astrom69} to prove convexity of value function. Proof is along lines of \cite[Lemma $2$]{Meshram18}.
We use maximum likelihood ratio (MLR) order for comparison of belief $\pi'$s   MLR order is denoted  as $\pi \geq_r \widetilde{\pi}.$   Totally positive order $2$ ($\mathrm{TP}_2$) for comparison of transition probability matrices. 
\begin{lemma}[Monotonicity of value function] 
	\cite{Lovejoy87}: 
	The optimal value function $V^*(\pi)$ is monotone in belief $\pi,$ that is, $V^*(\pi) \geq V^*(\widetilde{\pi})$ whenever $\pi \geq_{r} \widetilde{\pi}$ for $\pi, \widetilde{\pi} \in \Pi(S)$ under following assumptions. 
	\begin{itemize}
		\item reward $r(j,a)$ is non decreasing in  $j \in \mathcal{S}$ for fixed $a.$
		\item transition probability matrices $P^1$ and $P^0$ are $\mathrm{TP}_2$ ordered.
		\item the observation row vector for arm $q(j) \geq q(k)$ for $j \geq i$ and $i,j \in \mathcal{S}.$ 
	\end{itemize}
	\label{lemma:monotonicity} 
\end{lemma} 

We  sketch the proof.  The assumptions stated here preserves monotonocity in belief $\Gamma,$ and  $\sigma$ whenever there is ordering in prior $\pi,$ action $a $ and observation $k.$ This  preserves the ordering in value functions in belief $\pi.$ Using induction method on dynamic program and monotonicity of value functions in belief , we get the desired result. 


We note that the Lemma~\ref{lemma:convexity} and ~\ref{lemma:monotonicity} holds for all models under different assumptions on model. But threshold policy and indexability holds true  only for model $2$ and $3.$ 	

A threshold type policy provides partition of belief state space $\Pi(S)$ into three disjoint regions,  $\Lambda_1, \Lambda_2, \Lambda_3 \subseteq \Pi(S),$  where $\Lambda_1 = \{\pi \in \Pi(S): a^*_t(\pi) = 1  \}$ $\Lambda_2 = \{\pi \in \Pi(S): a^*_t(\pi) =  0 \}$ and  $\Lambda_3 = \{\pi \in \Pi(S): a^*_t(\pi) = 1~\text{and }~ 0 \}.$ $a^*_t(\pi) \in \{0,1\}$ is the optimal action for belief $\pi$ at time step $t.$   Illustration of this is given in Fig.~\ref{fig:threshold-type}.

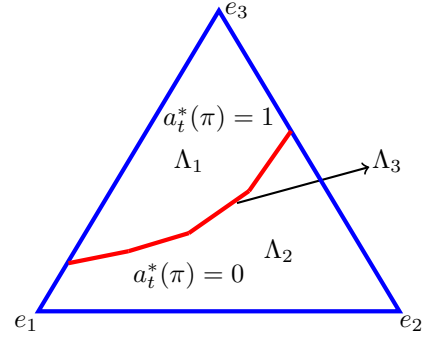
\begin{figure}
	\begin{center}  
		\begin{tikzpicture}[scale = 0.8] 
		\draw (3, 3.2) node {$a^*_t(\pi) = 1$};
		\draw (2.5, 0.6) node {$a^*_t(\pi) = 0$};
		\draw (-0.2,-0.2) node {$e_1$};
		\draw (6.2,-0.2) node {$e_2$};
		\draw (3.3,5) node {$e_3$};
		\draw (2.5,2.5) node {$\Lambda_1$};
		\draw (4,1) node {$\Lambda_2$};
		\draw (5.8,2.5) node {$\Lambda_3$};
		\draw[black, thick, ->] (3.3,1.8) -- (5.5,2.4);
		\draw[blue, ultra thick] (0,0) -- (6,0) -- (3,5) -- cycle;
		\draw[red, ultra thick] (0.5,0.8) -- (1.5,1);
		\draw[red, ultra thick] (1.5,1) -- (2.5,1.3);
		\draw[red, ultra thick] (2.5,1.3) -- (3.5,2);
		\draw[red, ultra thick] (3.5,2) -- (4.2,3);	
		\end{tikzpicture}
	\end{center}
	\caption{Threshold type policy illustration}
	\label{fig:threshold-type} 
\end{figure}

\begin{definition}[Threshold type policy]	
	The optimal policy is called a threshold type if one of the following holds true. 
	\begin{enumerate}
		\item The optimal action $a_t^*(\pi) =1$ for all $t$ and all $\pi \in \Pi(S),$ that is $\Lambda_1 = \Pi(S)$ and $\Lambda_2 = \Lambda_3 = \emptyset.$
		\item The optimal action $a_t^*(\pi) =0$ for all $t$ and all $\pi \in \Pi(S),$ that is $\Lambda_2 = \Pi(S)$ and $\Lambda_1 = \Lambda_3 = \emptyset.$
		\item The optimal action $a_t^*(\pi) =1$ for  all $\pi \in \Lambda_1,$ $a_t^*(\pi) =0$ for  all $\pi \in \Lambda_2,$ and $a_t^*(\pi) =1$ and $0$ for  all $\pi \in \Lambda_3,$ that is, $\Lambda_1, \Lambda_2 ,\Lambda_3 \neq  \emptyset.$ Also $\Lambda_1 \cap \Lambda_2 \cap \Lambda_3 = \emptyset.$ 
	\end{enumerate}
\end{definition} 
We next show a threshold policy result and indexability for model $2$ and $3.$ 
We make use of same assumption as stated in previous Lemma~\ref{lemma:monotonicity}.
\begin{lemma}[Threshold type policy]
	In Model $2$ and Model $3,$ the  optimal policy is of  threshold type.
	\label{lemma-threshold-policy} 
\end{lemma}

We provide a sketch of the proof. Define $f(\pi, V^*) = g(\pi,a=1, V^*) -
g(\pi,a=0, V^*).$ We show that $f(\pi, V^*)$ is non decreasing in $\pi.$ 
In these model, not playing action, i.e., $a=0$ implies restart state where transition occurs to  a fixed state. Thus the future value function for action $a-0$  is constant.  From definition of $f(\pi, V^*),$ that term gets canceled, hence using  Lemma~\ref{lemma:monotonicity}, we show that $f(\pi, V^*)$ is non decreasing in $\pi.$ This is sufficient for threshold type policy. Detailed proof is given in Appendix.

\subsection{Indexability and Whittle index}
From threshold policy result in Lemma~\ref{lemma-threshold-policy}, we  define 
\begin{eqnarray*}
	U_{1}(W) := \left\{ \pi \in \Pi(S): V(\pi, a =1, W)  > V(\pi, a = 0,W ) \right\} 
	\nonumber \\
	U_{0}(W) := \left\{ \pi \in \Pi(S): V(\pi, a =1, W)  \leq  V(\pi, a = 0,W ) \right\} 
\end{eqnarray*}
Hence  $U_{0}(W) = \Lambda_2 \cup \Lambda_3.$ 
\begin{definition}[Indexability \cite{Whittle88}]
	As subsidy $W$ increases from $-\infty$ to $+\infty,$  $U_{0}(W)$ increases from $\emptyset$ to full set $\Pi(S).$ 
	\label{def:indexability-single-dim}	
\end{definition} 
To show the indexability we require that whenever $W_2 > W_1$ implies $U_0(W_1) \subseteq U_0(W_2).$ We use the following result for indexability. 
\begin{lemma}
	For $\pi \in \Pi(S)$ if 
	\begin{equation}
	\frac{\partial V(\pi, 1, W)}{\partial W} \bigg 
	\rvert_{\pi = \pi_T(W)} \ < \
	\frac{\partial V(\pi, 0, W)}{\partial W} \bigg 
	\rvert_{\pi=\pi_T(W)},
	\label{eq:grad-VS-less-grad-VNS1}
	\end{equation}
	and  $\pi_T(W) \in \Lambda_3, $
	then $U_0(W)$ is a monotonically increasing function of
	$W.$
	\label{lemma:indexability}
\end{lemma}
Proof of this lemma is analogous to \cite[Lemma $4$]{Meshram18}. 
We now present main result.
\begin{theorem}[Indexable]
	The single-armed restless hidden Markov bandit is indexable for $0 < \beta< 1$ and
	$W_{a} \leq W \leq W_b.$
	\label{thm:indexability}
\end{theorem}
%
\begin{proof}
	From Definition \ref{def:indexability-single-dim}, we need to show that 
	$U_0(W_1) \subseteq U_0(W_2)$ whenever $W_2 > W_1.$
	Note that  $V(\pi, a=1,W) -V(\pi,a =0,W)$ is decreasing
	in $W$\footnote{By induction method, one can show that $V(\pi, a=1,W)$ is non decreasing $W$ and $V(\pi, a=0,W)$  is strictly increasing in $W$ for fixed $\beta$ and $\pi$} for fixed $\pi, \beta.$ Therefore, equation
	\eqref{eq:grad-VS-less-grad-VNS1} holds true. Using Lemma
	\ref{lemma:indexability},   $U_0(W_1) \subseteq U_0(W_2).$  
	whenever $W_2 > W_1$ and $W_1, W_2 \in [W_a, W_b].$
	This completes the proof.
\end{proof}
%
We next define the Whittle index.
\begin{definition}[Whittle index \cite{Whittle88}]
	If an arm is indexable and is in state $\pi \in \Pi(S),$ then its Whittle
	index, $W(\pi),$ is
	$		W(\pi) := \inf_{W}\{W: V(\pi,1, W) =
	V(\pi,0,W) \}.$
	\label{def:whittleind} 
\end{definition}

$W(\pi)$ is a minimum subsidy $W$   such that the optimal action
is not to play the arm at given $\pi.$  The Whittle index formula requires explicit expression of $V(\pi,1,W)$ and
$V(\pi,0,W).$ Then we have  to equate  and solve this for $W.$  For Model $2,$ the index formula is not feasible but we will provide approximate index computation algorithm.  
For Model $3,$ we obtain closed form expression of index and this is given in next lemma.

\begin{lemma}[Whittle index formula for model $3$]
	Whittle index  for given belief $\pi$ is computed based on region of $e_{m_1}$ and $e_{m_2}.$   We assume that $m_1 > m_2.$
	\begin{itemize}
		\item if $e_{m_1} \in U_1(W)$  and $e_{m_2} \in U_0(W),$ then 
		\begin{eqnarray*}
		W(\pi) = (1-\beta) \left(\sum_{j=1}^{n} \left[r(j,1) - r(j,0)\right] \pi(j) \right) + \\
		\beta \left[ r(m_1,1) - r(m_2,0) \right] 
		\end{eqnarray*}
	   \item if $e_{m_1},e_{m_2} \in U_1(W)$ then 
	   \begin{eqnarray*}
	   	W(\pi) =  \sum_{j=1}^{n} \left[r(j,1) - r(j,0) \right] \pi(j) + \\
	   	\beta [r(m_1,1) - r(m_2,1)]
	   \end{eqnarray*}
	   \item if $e_{m_1},e_{m_2} \in U_0(W)$ then
	   \begin{eqnarray*}
	   	W(\pi) = \sum_{j=1}^{n} \left[r(j,1) - r(j,0)\right] \pi(j) + \\
	   	\beta \left[ r(m_1,0) -r(m_2,0)\right].
	   \end{eqnarray*}
	\end{itemize} 
\label{lemma:index-comp-model3}
\end{lemma}
Proof of this is given in Appendix. When there is no reward from not playing except subsidy $W,$ we can have $r(j,0) = 0$ for all $j \in \mathcal{S}.$

\section{Monte Carlo rollout policy}
\label{sec:Monte-Carlo-Rollout}
We now discuss Monte Carlo rollout policy algorithm for a single-armed bandit in case of Model $2.$  Algorithm is based on simulations, where initial belief state $\pi$ and a subsidy $W$ is given as input. We run multiple-trajectories, and each trajectory consists of (belief state $\pi$, action $a$, and observed reward $r$) 
Thus the information obtained from a single trajectory upto horizon length $H$ is  $\{ \pi_{1,l},a_{1,l}, r_{1,l}, \pi_{2,l}, a_{2,l},r_{2,l}, \cdots, \pi_{H,l}, a_{H,l}, r_{H,l} \}$ under policy $\phi.$ Here, $l$ denotes a trajectory. The value estimate of $k$th trajectory starting from belief state $\pi,$ action $a= 1$ and action $a= 0$ is 
\begin{eqnarray*}
	Q_{H,l}^{\phi}(\pi, a, W) &=& \sum_{h=1}^{H} \beta^{h-1} r_{h,l}^{\phi} \\
	&=& \sum_{h=1}^{H} \beta^{h-1} r(\pi_{h,l},a_{h,l}).
\end{eqnarray*}  
Then value estimate for state $\pi$ and action $a$ over $L$ trajectories under policy $\phi$ is 
\begin{eqnarray*}
	\widetilde{Q}_{H,L}^{\phi}(\pi, a, W) = \frac{1}{L}\sum_{l=1}^{L}  Q_{H,l}^{\phi}(\pi, a, W).
\end{eqnarray*} 
The output of Monte Carlo algorithm is $\widetilde{V}_{\phi,H,L}(\pi,a= 1,W)$ and $\widetilde{V}_{\phi,H,L}(\pi,a= 0,W).$ 
\begin{eqnarray*}
\widetilde{V}_{\phi,H,L}(\pi,a= 1,W) = r(\pi,a=1) + \widetilde{Q}_{H,L}^{\phi}(\pi,a =1,W)  
\end{eqnarray*}
\begin{eqnarray*}
\widetilde{V}_{\phi,H,L}(\pi,a= 0,W) =   W + r(\pi,a=0)  + \\  \widetilde{Q}_{H,L}^{\phi}(\pi, a=1, W) 
\end{eqnarray*}

%

\subsection{Index computation for Model $2$ }
 \label{sec:index-compute-algo}
We present algorithm for Whittle index computation using Monte Carlo rollout policy. It is described in Algorithm~\ref{algo:Whittle-index-compute}. Input is state $\pi,$ and initialize value $W.$ We run Monte Carlo rollout policy under threshold policy $\phi$ for $W$ and state $\pi.$ We obtain approximate value functions  $\widetilde{V}_{\phi,H,L}(\pi ,a= 1, W)$ and $\widetilde{V}_{\phi,H,L}(\pi,a= 0,W).$ If the difference between these approximate value functions is higher than $\epsilon > 9,$ then we change $W$ to new value of $W;$ otherwise exit an algorithm with output index $=W.$   The convergence of this algorithm follows from two-timescales stochastic approximation algorithms, \cite[Chapter $6$]{Borkar08}. In our setting, Monte Carlo rollout policy algorithm runs on faster timescale and the subsidy $W$ is updated on slower timescale. We use $\gamma$ as learning rate for $W.$  

\begin{algorithm}
	\caption{Whittle index computation algorithm for an arm}
	\begin{algorithmic}
		\STATE \textbf{Input: State of arm}  $\pi$ \\
		\STATE \textbf{Initialize} $W_{old} = W,$ $\epsilon =0.05$   $\Delta=1,$ and Stepsize $\gamma$
		\STATE \textbf{Define: } $W_{new} = W_{old}$
		\STATE \textbf{While} ($\Delta > \epsilon$) 
		
		\STATE \hspace{0.5cm} \textbf{1. Use Monte Carlo rollout policy} 
		\STATE \hspace{1cm} \textbf{Compute:} $\widetilde{V}_{\phi,H,L}(\pi,a= 1,W_{new})$ and \\ 
		\hspace{2cm} $\widetilde{V}_{\phi,H,L}(\pi,a = 0,W_{new})$  
		\STATE \hspace{0.5cm} \textbf{2. Define} 
		\begin{eqnarray*}
			\Delta(\pi,W_{new}) = \widetilde{V}_{\phi,H,L} (\pi,a= 1,W_{new}) - \\ \widetilde{V}_{\phi,H,L}(\pi,a = 0,W_{new}) 
		\end{eqnarray*}
		\STATE \hspace{1cm} $W_{old} = W_{new}$
		\STATE \hspace{1cm} $W_{new} = W_{old} + \gamma \Delta(\pi,W_{new})$
		\STATE \textbf{End}
		\STATE \textbf{3. Output:} $W(\pi)$ 
	\end{algorithmic}
	\label{algo:Whittle-index-compute}
\end{algorithm}

We derive following result with Monte Carlo rollout policy  assuming there optimal policy exists and it of  threshold type, say, $\phi.$ 
\begin{theorem} 
We assume that $r(s,a) \in [0,1],$ $0 \leq W \leq 1.$	For sufficiently large horizon length $H,$ there exist number $\widetilde{L}$  such that for all $L > \widetilde{L}$ we have 
	with probability $1-\frac{2}{H^2}$  
	\begin{eqnarray*}
		\bigg \vert V_{\phi}(\pi, a, W) - \widetilde{V}_{\phi,H,L}(\pi,a,W) \bigg \vert \leq 
		\sqrt{\frac{ z^2 \log(H) }{ L}}
	\end{eqnarray*}
	for $a \in \{0,1\}.$
	Here, $z = \frac{(1-\beta^{H})}{1-\beta}.$
	\label{thm:concentration-MC}
\end{theorem}
We discuss the proof idea. We simulate $L$ number of trajectories which are  are independent and cumulative reward collected along each trajectory is random.  Trajectories are generated using a fixed policy $\phi.$ We use Hoeffding inequality \cite{Hoeffding63}. The probability of deviation between the infinite horizon discounted value function under policy $\phi$ and estimated value function obtained using over $L$ number of simulated trajectories greater than confidence bound decays exponentially fast. After  simplifications we obtain desired result. 
Detail steps are given in Appendix. 

\subsection{Monte Carlo rollout policy for Model $1$} 
As discussed in earlier section index policy is not applicable to Model $1,$ however we can use Monte Carlo rollout policy. Here, arm is selected based on state-action value estimate obtained using fixed Rollout policy $\phi$ that selects an arm at each time step. Note that we are directly applying this policy to RMAB. 

Detail of rollout policy is as follows. There are $L$ trajectories simulated for a fixed horizon length $H$ using a known transition and reward model. Along each trajectory, a fixed policy $\phi$ is employed according to which one arm  is played at each time step from $N$ arms. The  information obtained from a single trajectory upto horizon length $H$ is
\begin{eqnarray} 
\{ \pi_{t,j,l},a_{t,j,l}, r_{t,j,l}^{\phi}\}_{j=1,t=1}^{ N,  H}
\end{eqnarray} 
under policy $\phi.$ Here, $l$ denotes a trajectory, the belief for arm $j$ is $\pi_{t,j,l} \in \Pi(\mathcal{S}),$ action of arm $j$ is $ a_{t,j,l} \in \mathcal{A},$ moreover it has constraint $\sum_{j=1}^{N}  a_{t,j,l} =1.$ $r_{t,j,l}^{\phi}$ is reward from arm $j$ under policy $\phi.$
The value estimate of  trajectory $l$ starting from belief state $\pi = (\pi_{1}, \cdots, \pi_N),$ and $\pi_j \in \Pi(\mathcal{S})$ for $N$ arms and initial action $\alpha \in \{1,2, \cdots,N \}$ and  is   
$	Q_{H,l}^{\phi}(\pi, \alpha) = \sum_{h=1}^{H} \beta^{h-1} r_{h,l}^{\phi} 
= \sum_{h=1}^{H} \beta^{h-1} r(\pi_{h,l},\alpha_{h,l}, \phi).$
Then, the value estimate for state $\pi$ and action $a$ over $L$ trajectories under policy $\phi$ is 
\begin{eqnarray*}
	\widetilde{Q}_{H,L}^{\phi}(\pi, \alpha) = \frac{1}{L}\sum_{l=1}^{L}  Q_{H,l}^{\phi}(\pi, \alpha).
\end{eqnarray*} 
We use  myopic (greedy) policy as base policy $\phi$ that is implemented for a trajectory.  One step policy improvement is performed, and the optimal action  is selected according follow rule. 

\begin{eqnarray}
j^*(\pi) = \arg \max_{1 \leq j \leq N} \left[ r(\pi, \alpha= j) + \beta \widetilde{Q}_{H,L}^{\phi}(\pi, \alpha = j) \right].
\end{eqnarray}
In each time step, an arm is  played based on the above rule.  Detailed discussion on rollout policy for multi-action RMAB and fully observable state is given in \cite{Meshram2020}. 
In next section we present  numerical examples using Monte Carlo rollout policy.

\section{Numerical Results and Discussion}
\label{sec:numerical-discussion}

We describe  three numerical examples that demonstrate  the performance of index policy, myopic policy and Monte Carlo rollout policy. In the myopic policy, the arm with highest immediate expected payoff is played at each time step. In index policy, the arm with highest index is played. 

We present first numerical example for  model $1.$ We use following parameters. The number of arms $N=15,$ number of states $n=4,$ discount parameter $\beta = 0.95,$ number of message $K =2$ and binary reward is considered for each state. Assume that the transition probabilities and observation probabilities are know. As the states are not observable at all in this model, we do not make assumption on  transition probabilities,  i.e. $\mathrm{TP}_2$ order.   We compare Monte Carlo rollout policy and myopic policy. We use number of horizon $H=5$ and number of trajectories $L=100.$ 
We plot iteration vs discounted cumulative reward. We  observe from Fig.~\ref{plots:Myopic-MC-Model1} that
Monte Carlo rollout policy performs better than myopic policy up to $25\%.$ Though rollout policy is computationally expensive it has advantages in terms of higher cumulative reward.  

\begin{figure}
	\begin{center}
			\includegraphics[width=0.9\columnwidth]{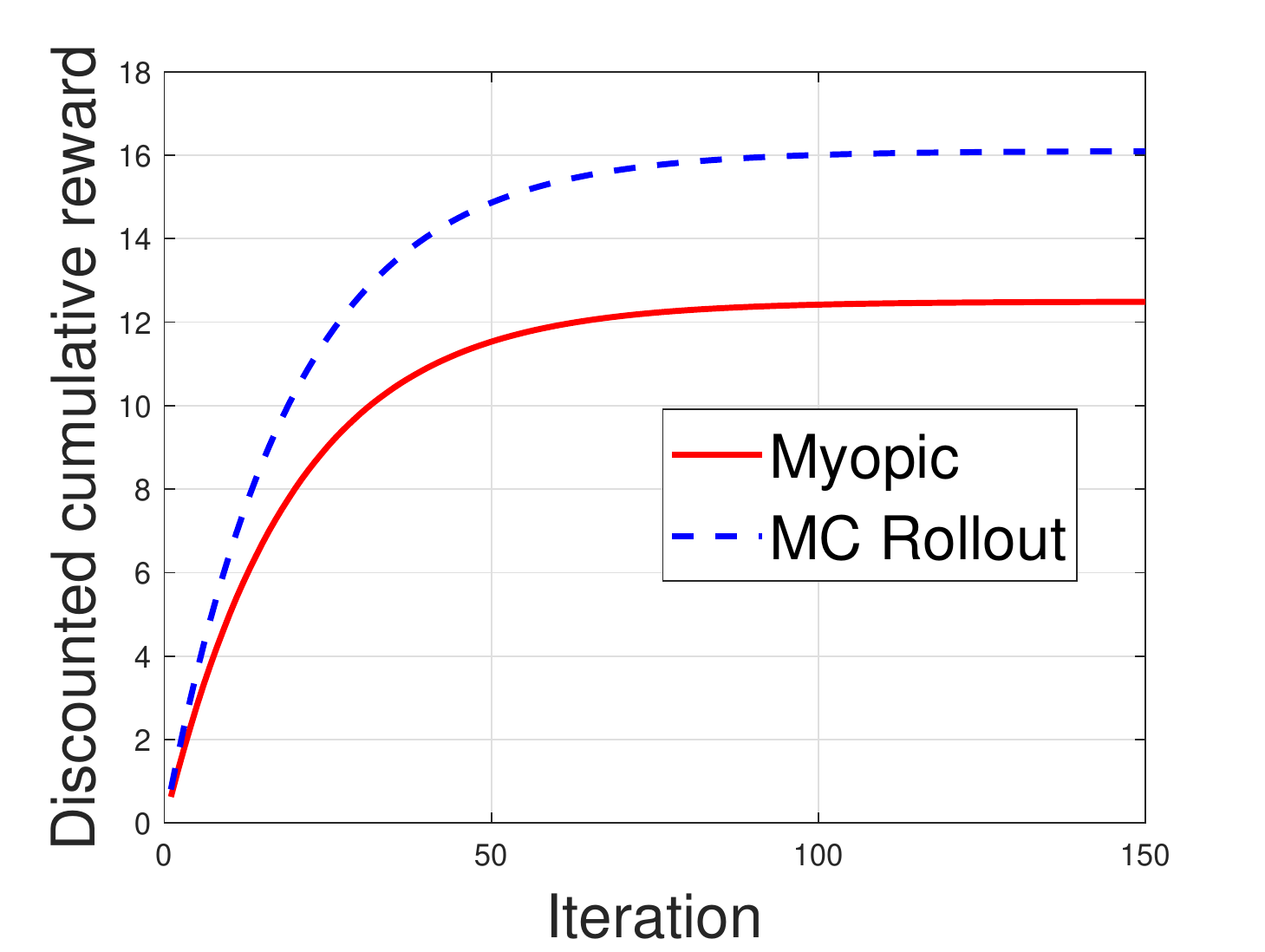}
	\end{center}
	\caption{Model $1:$ Myopic vs Monte Carlo rollout policy}
	\label{plots:Myopic-MC-Model1}
\end{figure}
 
In our second example is for model $2,$ where we consider number of arms $N=5,$ number of states $n=4,$ $\beta = 0.95,$ $K=2$ and binary reward is obtained from each state after play of arm and no reward is obtained after not playing of arm. In this example we compare index policy and myopic policy. We note that index computation is performed using Monte Carlo rollout policy, where we use $H=5.$ We observe from Fig.~\ref{plots:Myopic-Index-Model2} that myopic policy performs better than approximate index policy based algorithm. Myopic performs better by  $5\%.$  This difference is due to approximation in index computation. 

\begin{figure}
	\begin{center}
		\includegraphics[width=0.9\columnwidth]{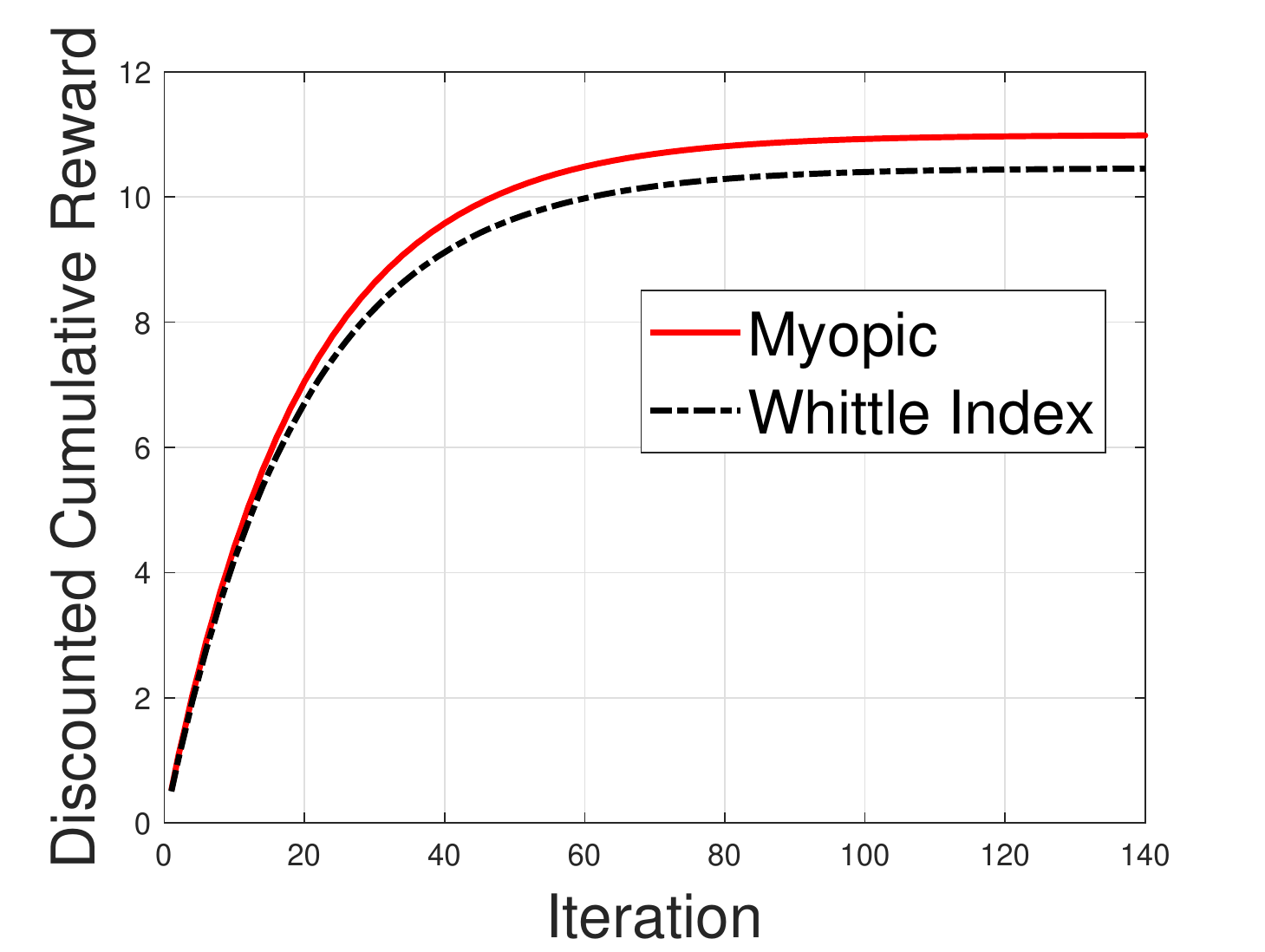}
	\end{center}
	\caption{Model $2:$ Myopic vs Approximate index policy}
	\label{plots:Myopic-Index-Model2}
\end{figure}

In our third example, we present numerical example for model $3.$ Here, Whittle index formula is explicitly available. We compare myopic policy and Whittle index policy for $N=15,$ $n=4$ and discount parameter $\beta = 0.95.$ We observe from Fig.~\ref{plots:Myopic-Index-Model3} that Whittle index policy performs poor that myopic policy. This is due to myopic policy plays only a fixed arm, $3$ for all times whereas Whittle index policy plays more than one arm more frequently based on index. In this example it suggest Whittle index policy is not optimal but it is fair and plays other arms as well.

\begin{figure}
	\begin{center}
		\includegraphics[width=0.9\columnwidth]{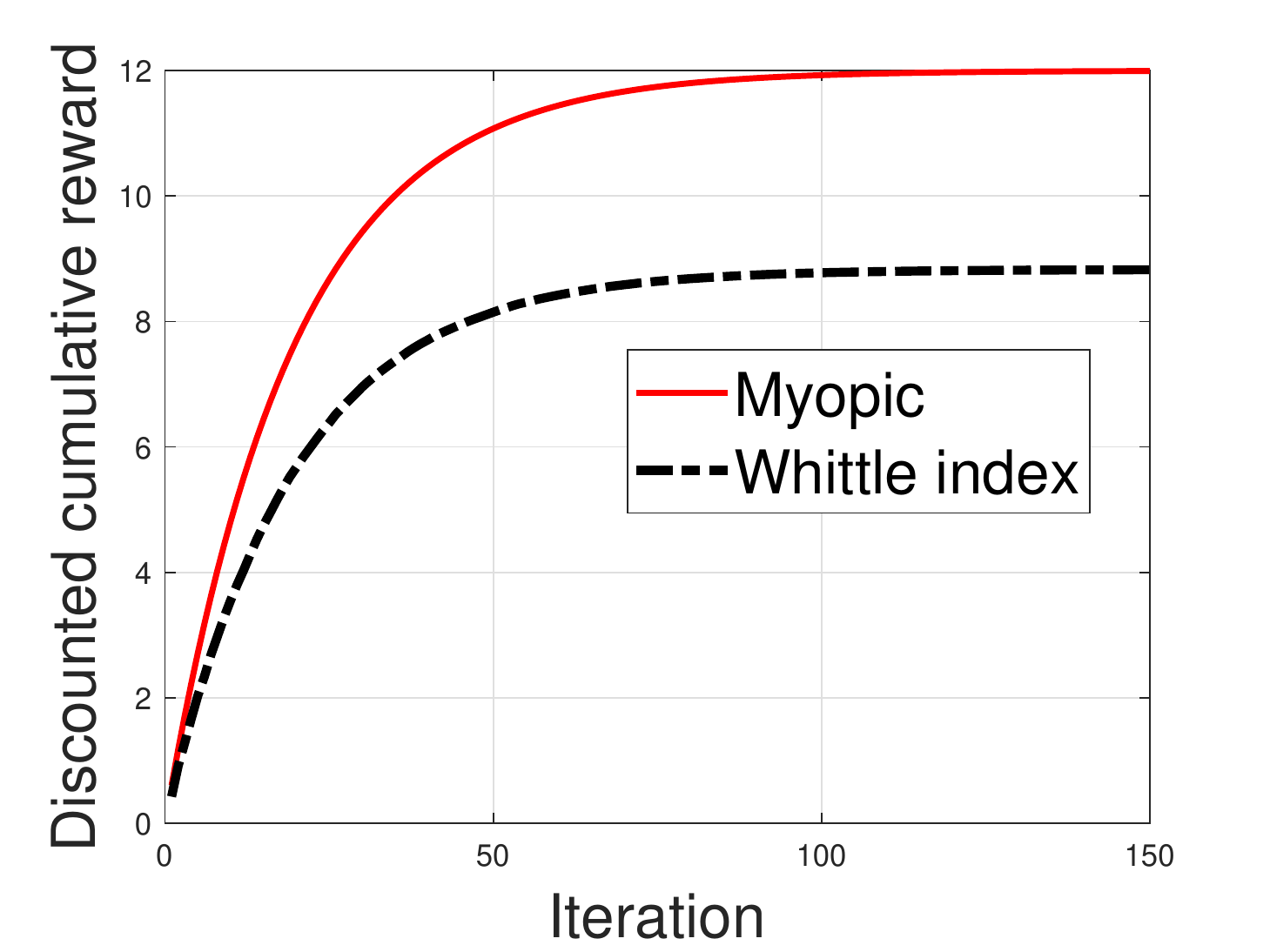}
	\end{center}
	\caption{Model $3:$ Myopic vs Whittle index policy}
	\label{plots:Myopic-Index-Model3}
\end{figure}

\section{Concluding Remarks}

In this paper we studied partially observable restless multi-armed bandits. We considered three different models based on information observable to decision maker.  

From numerical examples, it suggests that application of directly Monte Carlo rollout policy on restless multi-armed bandits can have advantages over myopic policy.  In general, an index policy for multi state partially observable models need not be optimal. We observed that Whittle index policy need not be optimal even we have index formula. A simple rollout policy is  competitive to myopic policy when no index formula is available.  

This opens interesting future direction of work on  MC rollout policy for other partially observable models when indexability and index computations are infeasible.

\bibliographystyle{IEEE}

\bibliography{restless-bandits}

\appendix


\subsection{Proof of Lemma~\ref{lemma-threshold-policy}}
\label{app:lemma-threshold-policy}
We now define $f$ for finite horizon as follows. 
\begin{eqnarray}
f(\pi,V^*_t) = g(\pi,a, V_t^*) - g(\pi,a^{\prime}, V_t^*) & \mbox{$a \geq a^{\prime},$ $a, a^{\prime} \in \mathcal{A}$ }
\label{eqn:diff-value-f}
\end{eqnarray}
We next show a threshold-type policy result and to claim this result, we require to show that $f(\pi, V^*_t)$ is nondecreasing in $\pi \in  \Pi(S).$ This property also referred to as submodularity of function. Even though optimal value function $V_t^*(\pi)$ is monotone in $\pi,$ we can not say about this difference for  model $1.$  To see this, we substitute value of $g$ in Eqn.~\eqref{eqn:diff-value-f}, then 
\begin{eqnarray}
f(\pi,V^*_t) = \sum_{i=1}^{n} \pi_i (r(i,a) -r(i,a^{\prime})) +  \nonumber \\
\beta \sum_{k \in O} \sigma(k~|~\pi,a) V_t^*(\Gamma(\pi,a,k)) - \nonumber \\
\beta \sum_{k \in O} \sigma(k~|~\pi,a^{\prime}) V_t^*(\Gamma(\pi,a^{\prime},k))
\label{eqn:diff-value-f-subst}
\end{eqnarray} 
Note that monotonicity of value function, we can say that term $1$ and term $2$ in  Eqn.~\eqref{eqn:diff-value-f-subst} is monotone but third term has  negative sign, which  introduces difficulty for threshold policy behavior.  

But in case of  model $2$ and $3,$ we can claim threshold policy result. Under structural assumption on model,  i.e., $\Gamma(\pi,a^{\prime},k) = e_{i}$ where $e_i$ is the unit vector of dimension $n$ with $1$ at $i$th position and zero at remaining position.   This simplifies the Eqn,~\eqref{eqn:diff-value-f-subst} as follows.
\begin{eqnarray}
f(\pi,V^*_t) = \sum_{i=1}^{n} \pi_i (r(i,a) -r(i,a^{\prime})) +  \nonumber \\
\beta \sum_{k \in O} \sigma(k~|~\pi,a) V_t^*(\Gamma(\pi,a,k)) - \nonumber \\
\beta  V_t^*(e_i)
\label{eqn:diff-value-f-subst-2}
\end{eqnarray} 
Now observe that third term is just constant and hence we can now  claim the monotonicity of $f(\pi,V^*_t)$ in $\pi$ under assumptions in Lemma~\ref{lemma:monotonicity}. This proves the threshold policy result. 

\qed 

\subsection{Proof of Lemma~\ref{lemma:index-comp-model3}}
\begin{itemize}
	\item We first derive index for  $e_{m_1} \in U_1(W)$  and $e_{m_2} \in U_0(W).$ We define the action value function $V_1(\pi) = V(\pi, a =1)$ and 
	  $V_0(\pi) = V(\pi, a =0).$
	\begin{eqnarray*}
	V(e_{m_1}) &=& V_1(e_{m_1}) \\
	V_1(e_{m_1}) &=& r(m_1,1) + \beta  V_1(e_{m_1}) .
	\end{eqnarray*} 
   Thus 
   \begin{eqnarray*}
   V(e_{m_1}) = \frac{r(m_1,1)}{1-\beta}.
   \end{eqnarray*}
  The action value function for action $1$ with belief $\pi$ is 
  \begin{eqnarray*}
  V_1(\pi) &=& \sum_{j=1}^{n} r(j,1) + \beta V(e_{m_1}) \\ 
   &=& \sum_{j=1}^{n} r(j,1) +  \beta \frac{r(m_1,1)}{1-\beta}.
  \end{eqnarray*}
 We now obtain the $V_0(\pi).$ 
 \begin{eqnarray*}
 V_0(e_{m_2}) &=& W + r(m_2,0) + \beta V(e_{m_2}) \\
 &=& \frac{W+ r(m_2,0)}{1- \beta}.
 \end{eqnarray*}
 \begin{eqnarray*}
 V_0(\pi) = W + \sum_{j=1}^{n} r(j,0 )\pi(j) + \beta V(e_{m_2}) \\
 = W + \sum_{j=1}^{n} r(j,0 )\pi(j) + \beta \left( \frac{W+ r(m_2,0)}{1- \beta}\right)
 \end{eqnarray*}
 From theshold policy we know that at $\pi$ we have $V_1(\pi) = V_0(\pi).$ After  equating and solving we get 
 \begin{eqnarray*}
 W(\pi) =(1-\beta) \left(\sum_{j=1}^{n} \left[r(j,1) - r(j,0)\right] \pi(j) \right) + \\
 \beta \left[ r(m_1,1) - r(m_2,0) \right] .
 \end{eqnarray*}
This is an index formula.
\item We now derive the index when $e_{m_1},e_{m_2} \in U_1(W).$ 
We obtain value function expression first. 
\begin{eqnarray*}
V(e_{m_1}) &=& V_1(e_{m_1}) \\
V_1(e_{m_1}) &=& r(m_1,1) + \beta V(e_{m_1}) \\
V(e_{m_1}) &=& \frac{r(m_1,1)}{1-\beta},
\end{eqnarray*}
and 
\begin{eqnarray*}
	V(e_{m_2})  &=& r(m_2,1) + \beta V(e_{m_1}) \\
	V(e_{m_1}) &=& r(m_2,1) + \beta \frac{r(m_1,1)}{1-\beta}. 
\end{eqnarray*}
Then 
\begin{eqnarray*}
V_1(\pi) &=& \sum_{j=1}^{n} r(j,1) \pi(j) + \beta V(e_{m_1}) \\
&=& \sum_{j=1}^{n} r(j,1) \pi(j) + \beta \frac{r(m_1,1)}{1-\beta}
\end{eqnarray*}
and 
\begin{eqnarray*}
V_0(\pi) =  W + \sum_{j=1}^{n} r(j,0) \pi(j) + \beta V(e_{m_2}) \\
= W + \sum_{j=1}^{n} r(j,0) \pi(j) + \beta \left[ r(m_2,1) + \beta \frac{r(m_1,1)}{1-\beta} \right]
\end{eqnarray*}
After equating $V_1(\pi)$ and $V_0(\pi)$ and solving for $W,$ we have 
\begin{eqnarray*}
W(\pi)  = \sum_{j=1}^{n} \left[r(j,1) - r(j,0) \right] \pi(j) + \\
\beta [r(m_1,1) - r(m_2,1)].
\end{eqnarray*}

\item We now derive index formula when $e_{m_1},e_{m_2} \in U_0(W).$ 
We obtain
\begin{eqnarray}
V(e_{m_2}) &=& V_0(e_{m_2}) \\ 
 &=& W + r(m_2, 0) + \beta V(e_{m_2}) \\
 &=& \frac{W + r(m_2, 0)}{1-\beta}.
\end{eqnarray}
 
\begin{eqnarray*}
V(e_{m_1}) = W+ r(m_1,0) + \beta  V(e_{m_2}).
\end{eqnarray*}
Then 
\begin{eqnarray*}
V_1(\pi) = \sum_{j=1}^{n} r(j,1) \pi(j) +  \beta V(e_{m_1}) \\
V_2(\pi) = W+ \sum_{j=1}^{n} r(j,1) \pi(j) +  \beta V(e_{m_2}) 
\end{eqnarray*}

After equating and solving these equations for $W,$ we obtain 
\begin{eqnarray*}
W(\pi) = \sum_{j=1}^{n} \left[r(j,1) - r(j,0)\right] \pi(j) + \\
\beta \left[ r(m_1,0) -r(m_2,0)\right].
\end{eqnarray*}

\end{itemize}
\qed 

\subsection{Proof of Theorem~\ref{thm:concentration-MC}}
\label{app:thm:concentration-MC}

%
%

Initial belief is $\pi_0 = \pi.$ The immediate expected reward at time $t$ for action $a=1$ is   $r(\pi_t, a = 1) = \sum_{i \in S} \pi_t(i) r(i,a=1).$ We have assumed $ 0<r(i,a=1) \leq  1$ Then immediate expected reward for action $a=1$ is bounded, and $ 0< r(\pi_t, a = 1) \leq 1$ and here $R_{\max} =1$ and $R_{\min} =0.$ Similarly the immediate expected reward for action $a =0$ is   $0<r(\pi_t, a = 0) + W < 1$ for any $\pi_t.$ We assume that $0 \leq W \leq 1.$

We suppose that $V_{\phi}(\pi, a, W)$ is the value function for an arm  under policy $\phi,$ with initial state $\pi,$ action $a$ and subsidy $W.$ 

Note that $\{Q_{h,l}^{\phi}(\pi, a, W)\}_{l=1}^L$ are independent random trajectories generated using policy $\phi$ for horizon length $H$ starting from state $\pi,$ action $a$ and subsidy $W.$
Thus, for each trajectory $l,$ we have $Q_{l,H}^{\phi}(\pi,a,W) \in \left[0, \frac{(1-\beta^{H})}{1-\beta}\right].$ This is due to reward is bounded in each steps by $R_{\max}=1.$  Let $z = \frac{(1-\beta^{H})}{1-\beta}.$

Define the action value function under policy $\phi$ is $Q^{\phi}(\pi,a,W)$  for starting belief $\pi$ and action $a.$ This is discounted cumulative expected reward for infinite horizon problem. 
Thus we utilize  the Hoeffding inequality~\cite{Hoeffding63}  for independent random bounded random variables. We have following inequality. 

\begin{eqnarray*}
	\prob{\bigg\vert Q^{\phi}(\pi,a,W) - \frac{1}{L} \sum_{l=1}^{L} Q_{H,l}^{\phi}(\pi, a, W) \bigg\vert > \epsilon }  \leq \\
	2\exp\left(- \frac{2 L^2 \epsilon^2 }{L z^2}\right)   
\end{eqnarray*}  

Thus RHS of preceding term is 
\begin{eqnarray*}
	2\exp\left(- \frac{2 L^2 \epsilon^2 }{L z^2}\right)    = 2\exp\left(- \frac{2 L \epsilon^2 }{ z^2}\right)  
\end{eqnarray*}
We want this term to $\delta$ Hence
\begin{eqnarray*}
	2\exp\left(- \frac{2 L \epsilon^2 }{ z^2}\right)   = \delta 
\end{eqnarray*}
After rearranging terms, we have 
\begin{eqnarray*}
	\epsilon = \sqrt{\frac{z^2}{2L} \log\left(2/\delta\right)} 
\end{eqnarray*}
Setting $\delta = \frac{2}{H^2}$ we get following inequality with probability $1- \frac{2}{H^2}$
\begin{eqnarray*}
	\bigg\vert Q^{\phi}(\pi,a,W) - \frac{1}{L} \sum_{l=1}^{L} Q_{H,l}^{\phi}(\pi, a, W) \bigg\vert \leq  \sqrt{\frac{z^2 \log H}{L} } 
\end{eqnarray*}
We know that 
\begin{eqnarray*}
	V_{\phi}(\pi, W) = \max_{a \in \{0,1\}} Q^{\phi}(\pi,a,W)	
\end{eqnarray*}
Thus we can have following inequality with high probabiliy $1-\frac{2}{H^2}$ for sufficiently large horizon $H$ and $L > \widetilde{L}$  
\begin{eqnarray*}
	\bigg\vert V_{\phi}(\pi, a, W) - \widetilde{V}_{\phi,H,L}(\pi,a,W) \bigg\vert \leq  \sqrt{\frac{z^2 \log H}{L} } 
\end{eqnarray*}
for $a \in \{0,1\}.$ 
This completes the proof. 

\qed

%

\end{document}